\documentclass[wcp]{jmlr}

\usepackage{longtable}

\usepackage{booktabs}
\usepackage{caption}
\usepackage{subcaption}
\newcommand\independent{\protect\mathpalette{\protect\independenT}{\perp}}
\def\independenT#1#2{\mathrel{\rlap{$#1#2$}\mkern2mu{#1#2}}}

\title[Degeneration in VAE]{Degeneration in VAE: in the Light of Fisher Information Loss}

    \author{\Name{Huangjie Zheng} \Email{zhj865265@sjtu.edu.cn}\\
    \addr Cooperative Medianet Innovation Center\\
  Shanghai Jiao Tong University, Shanghai, 200240, China
    \AND
    \Name{Jiangchao Yao} \Email{sunarker@sjtu.edu.cn}\\
    \addr Cooperative Medianet Innovation Center\\
  Shanghai Jiao Tong University, Shanghai, 200240, China
    \AND
    \Name{Ya Zhang} \Email{ya\_zhang@sjtu.edu.cn}\\
    \addr Cooperative Medianet Innovation Center\\
  Shanghai Jiao Tong University, Shanghai, 200240, China
      \AND
      \Name{Ivor W. Tsang} \Email{ivor.tsang@uts.edu.au}\\
      \addr Centre for Artificial Intelligence \\
  University of Technology Sydney, Sydney, Australia
  }

\begin{document}

\maketitle

\begin{abstract}
While enormous progress has been made to Variational Autoencoder (VAE) in recent years, similar to other deep networks, VAE with deep networks suffers from the problem of degeneration, which seriously weakens the correlation between the input and the corresponding latent codes, deviating from the goal of the representation learning. To investigate how degeneration affects VAE from a theoretical perspective, we illustrate the information transmission in VAE and analyze the intermediate layers of the encoders/decoders. Specifically, we propose a Fisher Information measure for the layer-wise analysis. With such measure, we demonstrate that information loss is ineluctable in feed-forward networks and causes the degeneration in VAE. We show that skip connections in VAE enable the preservation of information without changing the model architecture. We call this class of VAE equipped with skip connections as SCVAE and perform a range of experiments to show its advantages in information preservation and degeneration mitigation.

\end{abstract}


\section{Introduction}
Variational Autoencoder (VAE) \citep{journals/corr/KingmaW13} is one representative generative model to combine variational inference with deep learning, and has shown great promise in the recent years. This is not only because of its strong ability to reason raw data with meaningful representation, providing possibilities for downstream works such as classification, generation, etc., but also because of its automatic feature learning and inference process with deep learning techniques.
 
Nowadays, many variants of VAE are proposed to concern the unsupervised latent representation learning to adapt to various tasks. One line of research works break the simple assumption on likelihood in primitive VAE models and introduce the autoregressive density to sequentially model the generation. For example, ~\cite{DBLP:journals/corr/GulrajaniKATVVC16,DBLP:journals/corr/OordKK16,DBLP:journals/corr/OordKVEGK16} propose PixelCNN and PixelRNN to reconstruct the image pixel by pixel, and utilize the contextual information to constrain the generation. Another line of research works pay attention to the expressiveness of the posterior modeled with VAE. To improve its power, \cite{DBLP:journals/corr/BurdaGS15,tomczak2017vae} introduce more complex hierarchical priors or transform simple priors to the complex ones by the normalization flow and its variants~\citep{rezende2015variational,kingma2016improving,sonderby2016ladder}. Although previous works make VAE more flexible to adapt various tasks, they also raise some new problems. A typical one is the degeneration when VAE is in the deeper architecture.

\begin{figure}[t]
 \centering
         \includegraphics[width=\columnwidth]{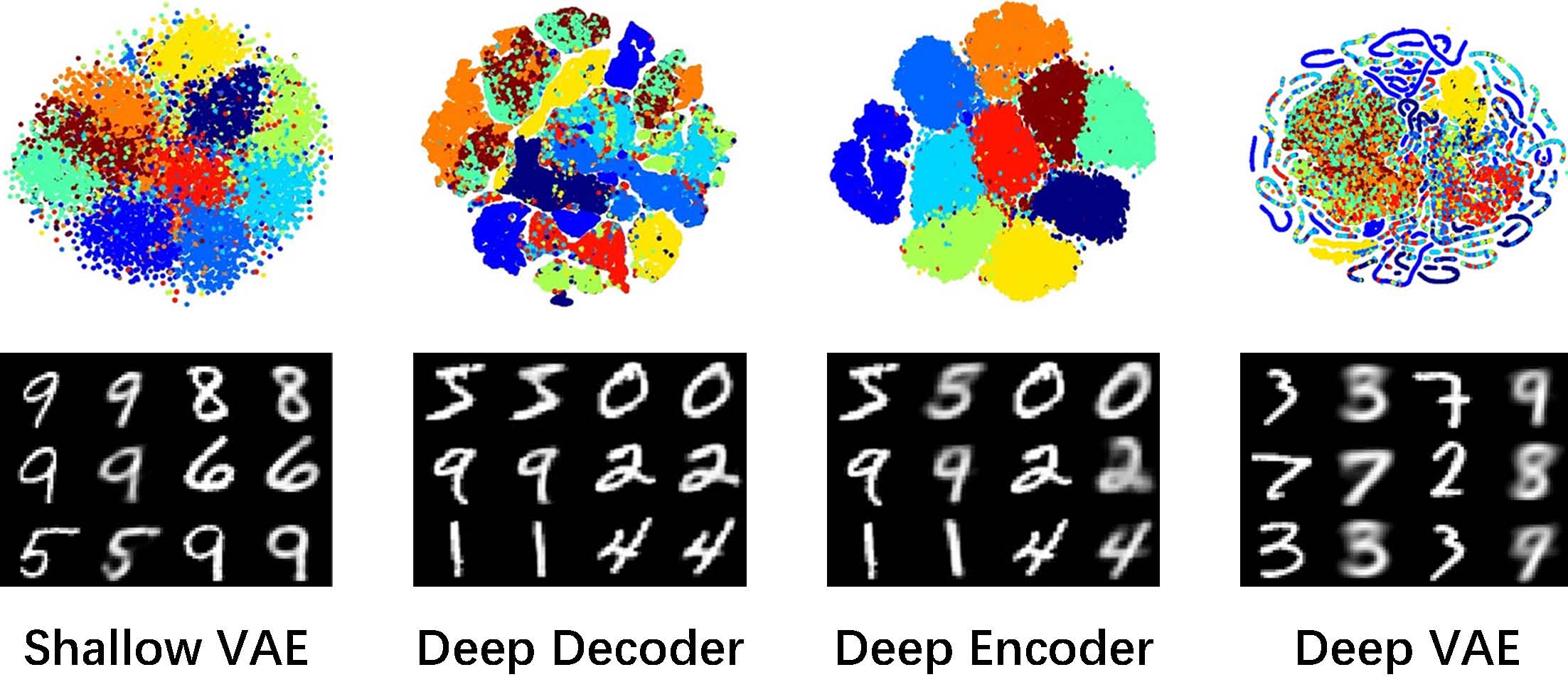}
     \caption{Deeper network's impacts in VAE models. \textbf{Upper}: representation of latent variable; \textbf{Lower}: ground truth (odd columns) and reconstruction samples (even columns).}
     \label{motiv}
\end{figure}

As expected, a deeper encoder should be conducive to the learning of useful latent code that well summarize the observations because of the powerful feature learning capacity of deeper feed-forward networks \citep{zeiler2014visualizing}; meanwhile, a deeper decoder should enable the production of generations of higher quality thanks to the better distribution modeling with deep networks \citep{DBLP:journals/corr/GulrajaniKATVVC16}. However, the degeneration is reported to limit the capacity of the deep networks \citep{saxe2013exact,he2016deep}. Different from the deep networks, in VAE, as shown in Figure \ref{motiv}, the degeneration occurs in three forms: Compared to VAE in shallow architecture, 1) a deeper decoder enables VAE to produce generations of higher quality, while the latent representation (visualized with t-SNE \citep{maaten2008visualizing}) is shown useless in providing high-level summary of the observation; 2) a deeper encoder helps the latent representation learn more global information, while the generation is of low quality; 3) when encoder and decoder both go deeper, VAE fails in both latent representation and generation. It seems that the degeneration does not affect VAE as the deep networks, but to harm the correlation between the input and the latent code.

The above phenomena motivate us to trace back to the connection between the input and the latent code. We illustrate the autoencoding process, in our paper, as a process of information transmission. Although in some previous work like \cite{zhao2017infovae} the mutual information has been proposed to enhance the connection between data and latent code, as the encoder and decoder go deeper, mutual information between data and latent code is more difficult to maintain. A natural solution is to investigate the information propagation layer by layer, which yet brings difficulty to the mutual information measure. Therefore, considering the output of hidden layers as a parametric implicit distribution, we propose a Fisher Information \citep{brunel1998mutual} measure to quantify the information loss layer by layer. With such measure, we demonstrate that the information loss generally exists in the encoder and decoder, which results in poor connection between latent code and data, thus leading to the previous three types of degeneration. In addition, we demonstrates that skip connections \citep{he2016deep,huang2016densely} could serve as a complementary information flow to help mitigate the degeneration without increasing the model complexity. Thus a variant SCVAE, \emph{i.e.}, VAE with skip connections, is proposed to preserve information when encoder and decoder go deeper. Finally, we conduct a series of experiments on widely used MNIST dataset. Comprehensive results indicate that our model performs well in information preservation, thus ensures a promising performance in latent representation learning and reconstruction at same time. Moreover, our model can be adaptive to other state-of-the-art VAE models for further amelioration.


\section{Related work}
In this section, we first review the recent progress of Variational AutoEncoder, which improves VAE in two perspectives: the expressiveness of likelihood and the expressiveness of posterior. Then we review some relevant works that address the degeneration of deep networks and the similarity to VAE. 

\subsection{Expressiveness of Likelihood}
Recently, research on a more expressive decoder has been conducted to improve the generative performance of VAE. 
In the research that applies VAE to sequence modeling, a powerful decoder is proved to be more expressive~\citep{chung2015recurrent}. Numerous research works combine recurrent and autoregressive models to achieve a powerful decoder: ~\cite{mathieu2015masked} proposed MADE, which masks the autoencoder’s parameters to respect
autoregressive constraints; ~\cite{gregor2015draw} proposed a recurrent structure to gradually reconstruct observations focusing on regions of interest. \cite{DBLP:journals/corr/GulrajaniKATVVC16,salimans2017pixelcnn++} model the dependencies among pixels with autoregressive density estimator \emph{e.g.}, PixelCNN and PixelRNN ~\citep{DBLP:journals/corr/OordKVEGK16,DBLP:journals/corr/OordKK16}  to serve as an expressive conditional distribution.  

\subsection{Expressiveness of Posterior}
Meanwhile, some works focus on augmenting the expressiveness of VAEs to model the complex posterior.  
In order to avoid using too simplistic priors, many research works propose to use multimodal distribution such as Gaussian mixture~\citep{dilokthanakul2016deep}. An alternative way is we first apply a simple prior, but complicate it gradually: normalization flow~\citep{rezende2015variational,kingma2016improving} is thus introduced to transform the variational distribution into more complex ones by applying successive invertible smooth transformation. Other methods design hierarchical latent variable structure~\citep{sonderby2016ladder} to approximate a complex posterior, or deploy auxiliary information such as label, ``Maximum Mean Discrepancy", pseudo-input, etc. to gradually increase the flexibility of posterior~\citep{DBLP:journals/corr/KingmaRMW14,louizos2015variational,tomczak2017vae}. 

\subsection{Degeneration in Networks}
It is noteworthy that neural networks and VAE possess similarities and differences when going deeper.   
\cite{he2016deep,huang2016densely} point out the vanishing-gradient problem that prevents the neural network from going deeper, and introduce residual connections to help training of very deep neural networks. \cite{saxe2013exact,orhan2017skip} point out deep neural network is defective by degeneration and claim degeneration occurs when networks lack sufficient information to provide for learning dynamics, which corresponds to the third degeneration observed in Figure \ref{motiv}. Apart from this phenomenon, the other two degeneration problems are similar to information preference \citep{chen2016variational,zhao2017infovae}. Different from that, deeper encoder or decoder is supposed to be a powerful approximator of distribution, but the expressiveness does not correspond to our expectation. Fisher Information can be applied to measure the quality of parameters in neural networks as mentioned in \cite{Desjardins:2015:NNN:2969442.2969471,Ollivier2015}. 
Inspired by these works, we investigate the observed problems in perspective of Fisher Information measure and demonstrate the existence of information loss in deep VAE. Skip connection is applied as a solution for information preservation, but not for avoiding gradient vanishing problem \citep{he2016deep}.

\section{Degeneration in VAE}\label{sec:degeneration}
In this section, we first give a brief review of VAE. Then we present the observed degeneration problems shown in Figure \ref{motiv}, which obstruct VAE from going deeper.

As we know, the goal of VAE is to reason the data $X$ with the latent variables $Z$ by marginalization \citep{journals/corr/KingmaW13}:
\begin{align} \label{eq:likelihood}
    \log{p_{\theta}(\textbf{x})} = \sum_{i=1}^N \log{\int p_{\theta}(x_i,z_i) d{z_i} } 
\end{align}
where $\theta$ is the parameter of the model and $N$ is the number of datdapoints. However, Eq.~\eqref{eq:likelihood} is usually intractable due to the lack of the analytical form for the integration. The common way to solve this problem is to introduce an evidence lower bound (ELBO):
\begin{align}\label{eq:elbo}
\begin{split}
\log{p_{\theta}(x_i)} 
& \geq \underbrace{\mathbb{E}_{q_{\phi}(z|x_i)}[\log{p_{\theta}(x_i|z)}] - D_{\text{KL}}(q_{\phi}(z|x_i)||p_{\theta}(z))}_{\mathcal{L}_{\text{ELBO}}(\theta,\phi;x_i)} ,\\
\end{split}
\end{align}
which is applied as an optimization objective so as to maximize the log-likelihood by introducing an inference model $q_\phi$ (also called recognition model) parameterized with $\phi$.
$\mathcal{L}_{\text{ELBO}}(\theta,\phi;x_i)$ consists of two terms: the first term is to fit the data, called \emph{reconstruction term}, and the remaining term is to fit the prior, called \emph{\text{KL}-divergence term}. When such lower bound is sufficiently optimized, the log-likelihood is approximately maximized. 

The advantage of VAE lies on combining the variational inference with deep learning. Networks are applied to model the posterior $q_{\phi}(z|x)$ and conditioned likelihood $p_\theta(x|z)$, named encoder, decoder respectively. When a network go deeper, the modeling capacity is supposed to be more powerful. Hence, the latent presentation and generation quality are supposed to be improved when VAE goes deeper.

However, this conjecture is not exactly in accord in the context of three types of degeneration shown in Figure \ref{motiv}. We observe the latent code (visualized by T-SNE \citep{maaten2008visualizing}) and the generation respectively. The shallow one is a typical Valina VAE \citep{journals/corr/KingmaW13}. We extend the depth of encoder/decoder of this referenced model. When the encoder is deepened, we observe that the visualization of latent code becomes more compact, which brings us an intuition that the model well summarizes high-level information, while the generation is of worse quality. When we only extend decoder depth, the result is in reverse. We observe generation of higher quality, while the latent code seems more abstruse. We thus expect to extend both sides to avoid this imbalance. Unfortunately, both latent code and generation become of worse quality.

Concretely, the third type of degeneration is equivalent to the degeneration in neural networks, {which occurs due to the lack of information for the learning in networks \citep{saxe2013exact}.} When VAE degenerates in this way, we can observe it is hardly optimized during training. The other two degeneration problems remind us of the information preference problem \citep{chen2016variational,zhao2017infovae}. {Different from our work, they use mutual information between data and latent code to enhance the meaningfulness of latent code. However, only referring to the mutual information between both ends is not enough as the architecture depth increases. The information transmission through the intermediate layers is worthy to concern. Moreover, in deep architecture, the mutual information is intractable for layer-wise computation, since the distribution form modeled by hidden layers is implicit, though parametric in most case. To address these concerns, we propose Fisher Information as a parametric measure, which will be discussed in the next section.}

\section{Fisher Information Loss Analysis}\label{sec:forget}
In this section, we first introduce the notion of Fisher Information, which is useful in information theory. Then we {illustrate the VAE as an information transmission process and analyze the above degeneration in the light of the Fisher Information.}

\subsection{Review of Fisher Information}
The Fisher Information is an important quantity in information theory and can be applied to {measure the quality of parametric estimation of distributions \citep{brunel1998mutual}.}

When we consider a stochastic variable $X$, whose probabilistic density function is $p_\theta(x)$, parameterized by $\theta \in \mathbb{R}$, we need to estimate the parameter $\theta$ from the measured values of the Variable $X$ (named \textit{observations}). Suppose that the {true} value of parameter is $\theta_0$. The estimation corresponds to the choice of the density $p_\theta(x)$ that minimize the relevant entropy \textit{w.r.t.} the true distribution $p_{\theta_0}(x)$ by a divergence:
\begin{equation}\label{eq:div}
    D_X(p_{\theta_0},p_\theta) = \mathbb{E}_{p_{\theta_0}}  \log \frac{p_{\theta_0}(X)}{p_\theta(X)}  = 
 \int_x  p_{\theta_0}(x) \log \frac{p_{\theta_0}(x)}{p_{\theta}(x)} dx
\end{equation}
 
In the information theory, the divergence in form \eqref{eq:div} is positive, convex and become zero when $\theta = \theta_0$. {Suppose its secondary derivative exists. When the divergence reach its optimal, the first order derivative is zero and the secondary derivative at $\theta_0$ is defined as Fisher Information \citep{brunel1998mutual}:}
\begin{equation} \label{eq:fidef}
  \mathcal{I}_p(\theta) = \frac{\partial^2}{\partial \theta^2}D_x(p_{\theta_0},p_\theta) = -\mathbb{E}_{X}\left[\frac{\partial^2 \log p_{\theta}(x)}{\partial \theta^2}\right]
\end{equation}
{Fisher Information is thus not a function \textit{w.r.t.} the stochastic variable $X$, but a function \textit{w.r.t.} the probabilistic density $p_\theta$ and useful for parametric estimation of distributions.}

One important characteristic of Fisher Information is that larger Fisher Information implies better understanding of the parameter, which facilitates the parameter estimation: Considering the curve of the divergence $D_X(p_{\theta_0},p_\theta)$, which is convex, larger Fisher Information makes the curve more ``steep" (\textit{e.g.} when $\mathcal{I}_p(\theta) = +\infty$, it becomes a \textit{dirac} centered on $\theta_0$), and it becomes easier to reach the optimal $\theta_0$. In this way, it reflects the quality of parameters regarding the approximation between modeled distribution $p_{\theta}(X)$ and the true distribution $p_{\theta_0}(X)$.

\subsection{ELBO as Information Transmission}
{For the simplicity and clarity of the formulation, we make two following assumptions:
First, we suppose all the stochastic variables are continue and have a probabilistic density. Second, we suppose that the probabilistic density functions are sufficiently regular, \textit{i.e.} they are continuously derivable and tend to zero at infinity (also for their first order derivative)\footnote{These hypothesis can be cancelled out by mathematical techniques to meet the request of the real-world situation.}.}

Suppose that our data $\chi$ is a set of samples, $\chi = \left\{x| x \in \mathbb{R}^d \right\}$, where $d$ is the dimension of one data sample. The measurable space of latent variable $Z$ is of dimension $s$, \textit{i.e.} $z \in \mathbb{R}^s$. According to the objective of \textbf{ELBO}, mentioned in Eq. \eqref{eq:elbo}, the VAE models and represents the data samples with the following process: 

$$
  \forall x \in \chi,\quad x \to z \to \widehat{x}
$$
where  $\widehat{x}$ is the sample reconstructed from the latent code $z$. {This process is implemented as an autoencoder in \cite{journals/corr/KingmaW13} but ignore the modeling of hidden layers.}

To be more detailed, the impacts of intermediate layers of encoder and decoder is naturally introduced. Therefore, by noting the output of the $l^{th}$ hidden layer as $h_l$ ($0<l<L$, L is the depth of network), the encoding (\emph{resp. decoding}) process can be illustrated as: 
$$ \text{encoding: } x \to h^{(en)}_1 \to h^{(en)}_2 \to \cdots \to z $$ 
$$ \text{decoding: } z \to h^{(de)}_1 \to h^{(de)}_2 \to \cdots \to \widehat{x} $$ 

Since the neural networks possess their probabilistic interpretation \citep{bishop2006pattern} (for example, the output of a MLP with linear activation can be interpreted as the mean of a conditional Gaussian distribution with a fixed variance \citep{pascanu2013revisiting}), we model the output of one hidden layer (\textit{e.g.} the $l^{th}$ layer) by using a stochastic variable $H_l$: 
$$ H_l \sim p_{\theta}(h_l|x, h_{1:l-1})$$
note that since the the network is not always a MLP, nor with linear activation, the distribution is not necessarily of Gaussian form, while can be regarded as implicit distribution. 

Therefore, based on the detailed auto-encoding process and the probabilistic view of hidden layers, the variational distribution $q_{\phi}(z|x)$ (\emph{resp.} generative distribution $p_{\theta}(\widehat{x}|z)$) can be reformulated as:
\begin{equation} \label{eq:detail_autoencoding}
\begin{split}
\text{encoding: } q_{\phi}(z|x,h^{(en)}_{1},h^{(en)}_2,\dots,h^{(en)}_{L-1})\\
\text{decoding: } p_{\theta}(\widehat{x}|z,h^{(de)}_1,h^{(de)}_2,\dots,h^{(de)}_{L-1})
\end{split}
\end{equation}

From the from \eqref{eq:detail_autoencoding}, we can learn that the information transmission is not only dependent on data $x$ and latent code $z$, but also the intermediate layers $h_{1:L}$. In previous work, such as \cite{zhao2017infovae}, Mutual Information is measured to reinforce the connection between $x$ and $z$. However, the relation between $x$ and $z$ can only partially reflects how information evolves through the hidden layers $h_{1:L}$ in the information transmission process. Plus, as the architecture becomes deeper, the measure between $x$ and $z$ is more complex to compute. These issues make the degeneration even harder to address. 

To address these concerns, we transform the non-parametric measure to parametric in order to investigate the layer-wise information. Thus, using Fisher Information, we can evaluate the information propagation quality through the encoder and decoder. Since Fisher Information is a parametric-wise measure, we can evaluate the quality of variational distribution and generative distribution \textit{w.r.t.} the network parameters $\phi$ and $\theta$. In the light of Fisher Information, we analyze the degeneration as a phenomenon of information loss, which will be discussed in the next part. 

\subsection{Degeneration Analysis with Fisher Information}
Fisher Information has been applied for efficient gradient backpropagation and the exact computation over layer-wise parameters can be achieved in neural networks \citep{Ollivier2015,Desjardins:2015:NNN:2969442.2969471}. In VAE models, encoding and decoding networks are applied to compute $q_{\phi}(z|x)$ and $p_{\theta}(x|z)$ \citep{journals/corr/KingmaW13}. We further generalize these distribution as $q_{\phi}(z|x,h^{(en)}_{1:L})$ and $p_{\theta}(x|z,h^{(de)}_{1:L})$, in consideration of the impacts of hidden layer output. To investigate the quality of parameters $\phi$ and $\theta$ in these distribution, Fisher Information is thus computed over parameters of the network and we discover the information loss, that causing the degeneration phenomena, as shown in the following proposition:
\begin{proposition}\label{prop:FI}
Suppose we apply a feed-forward network $F$ with $L$ layers to approximate a distribution $p(z|x)$, parameterized by $\Phi=\{\phi_1,\phi_2, \dots, \phi_L\}$:
$$Z \sim p_{\Phi}(z|x, h_{1:K}) =  F(x) = f_{\phi_L} \circ f_{\phi_{L-1}} \circ \cdots \circ f_{\phi_1}(x)$$
the network represents an approximated distribution $p_{\Phi}(x)$.
To evaluate the information evolution through the network, the compute Fisher Information in $l^{th}$ and $(l+1)^{th}$ layer can be computed and deduced as:
\begin{equation}\label{eq:Fisher}
\mathcal{I}_F(\phi_{l+1}) = \mathcal{I}_F(\phi_l) \left(\frac{\nabla_{\phi_l}F(X)}{\nabla_{\phi_{l+1}}F(X)}\right)^2
\end{equation}
where $\nabla_{\phi_l}F$ is the backpropagated gradient through the $i^{th}$ non-linearity. Note that the network $F$ can be either encoder or decoder. The corresponding input (\emph{resp.} parameter) can correspondingly be $x$ or latent code $z$ (\emph{resp.} $\Phi$ or $\Theta$). 
\end{proposition}
\begin{proof}
By definition in Eq. \eqref{eq:fidef}, {the Fisher Information passed through} the $l^{th}$ layer can be written as: 
$$\mathcal{I}_F(\phi_l) = -\mathbb{E}_{Z}\left[\frac{\partial^2 log F(x)}{\partial \phi_l^2}\right]$$
Similarly, for layer $l+1$, we compute the Fisher Information by definition and have:
\begin{align} 
 \mathcal{I}_F(\phi_{l+1}) 
&= -\mathbb{E}_{Z}\left[\frac{\partial^2 log F(x)}{\partial \phi_{l+1}^2}\right] \nonumber\\ 
&= -\mathbb{E}_{Z}\left[\frac{\partial}{\partial \phi_{l+1}}\left(\frac{\partial log F(x)}{\partial \phi_{l}} \cdot \frac{\partial \phi_l}{\partial \phi_{l+1}}\right)\right] \nonumber\\
&= -\mathbb{E}_{Z}\left[\frac{\partial^2 log F(x)}{\partial \phi_l \partial \phi_{l+1}}\cdot \frac{\partial \phi_l}{\partial \phi_{l+1}} + 
\frac{\partial^2 \phi_l}{\partial \phi_{l+1}^2} \cdot \frac{\partial log F(x)}{\partial \phi_l}\right] \nonumber\\
&= -\mathbb{E}_{Z}\left[\frac{\partial^2 log F(x)}{\partial \phi_l^2}\right] \cdot \left(\frac{\partial \phi_l}{\partial \phi_{l+1}}\right)^2 -\mathbb{E}_{Z}\left[\frac{\partial log F(x)}{\partial \phi_l}\right] \cdot
\frac{\partial^2 \phi_l}{\partial \phi_{l+1}^2} \nonumber \\ 
&= \mathcal{I}_F\left(\phi_l\right) \cdot \left(\frac{\partial \phi_l}{\partial \phi_{l+1}}\right)^2 -\mathbb{E}_{Z}\left[\frac{\partial log F(x)}{\partial \phi_l}\right] \cdot
\frac{\partial^2 \phi_l}{\partial \phi_{l+1}^2} \label{eq:proof1}
\end{align}

In Eq. \eqref{eq:proof1}, the term $\mathbb{E}_{Z}\left[\frac{\partial log F(x)}{\partial \phi_l}\right]$ is proved zero in many works of information theory \citep{brunel1998mutual,LY201740}. Thus we can only consider the first term. Additionally, between epoch $\{t,t+1\}$ in gradient descend optimization, $\phi_l^{(t+1)} = \phi_l^{(t)} + \lambda \cdot \nabla_{\phi_l}F$, where $\lambda$ is the learning rate.
Then we have $\partial \phi_l=\lambda \cdot \nabla_{\phi_l}F$.
Finally we have: 
$$\mathcal{I}_F(\phi_{l+1}) = \mathcal{I}_F\left(\phi_l\right) \left(\frac{\nabla_{\phi_l}F(X)}{\nabla_{\phi_{l+1}}F(X)}\right)^2.$$
\end{proof}
Using proposition \ref{prop:FI}, we can interpret the information transmission through the network by evaluating the gradient propagated through the network as a remark: 

\begin{remark}
Many works such as \cite{saxe2013exact} have reported that the gradient tends to get smaller as we move backward through the hidden layer. 

Using Eq. \eqref{eq:Fisher}m we have thus:
	$$ \mathcal{I}_F(\phi_{l+1}) \leq \mathcal{I}_F\left(\phi_l\right) $$
   which indicates that deeper layers tend to obtain less information layer by layer.
\end{remark}

It is interesting to notice the difference between a typical feed-forward neural network and VAE in perspective of information loss. Actually, information loss widely exists in deep neural networks, but it is often ignored. In fact, deep networks is powerful in learning hierarchical features\citep{zeiler2014visualizing}, with the learning process that tends to make features compact and discard superfluous information. This process is widely tolerated in networks' tasks though risky in degeneration in some cases.

{However, VAE cannot \textbf{simply} go deeper as networks. The loss of information makes VAE difficult to reach the true parameter $\phi_0$ and $\theta_0$. Recall the \textbf{ELBO} in Eq. \eqref{eq:elbo}, both two terms show the dependency between $q_\phi$ and $p_\theta$: the reconstruction request to compute the expectation of $\log p_\theta(x|z)$ \textit{w.r.t.} $q_\phi(z|x)$; meanwhile, the KL divergence connect $p_\theta$ and $q_\phi$. Therefore, either inaccurate parameter estimation of $q_\phi$ or $p_\theta$ will mislead the model's learning balance between $q_\phi$ and $p_\theta$. As a result, when facing the information loss, VAE needs to pick a choice between useful latent code and high-quality generation. }



\section{Fisher Information Preservation}\label{sec:info}
As discussed, the information loss is ineluctable in VAE and causes degeneration in deep architecture. A natural solution is thus preserving information without changing the parameter structure. In this section, we propose a simple but effective way for information preservation in VAE. 

{The skip connections \citep{he2016deep} can skip one or more layers of nonlinear mapping without changing the parameter dimension. Moreover, we demonstrate them as complementary information flows in this section. Thus, we propose a class of VAE equipped with skip connections, named SCVAE, to preserve the information in this way.}

As discussed, the output of a hidden layer in the neural network can be regarded as a stochastic variable and has a probabilistic distribution. Formally, we pose the stochastic variable $H_l$ to model the output of $l^{th}$ hidden layer, whose probabilistic density function is an implicit density function represented by the network:
\begin{equation}\label{eq:subnet}
    H_{l} \sim p_{\Phi}(h_{l}|x,h_{1},h_2,\dots,h_{l-1}) 
\end{equation}

When equipped with a skip connection (we use the stochastic variable $C$ to present) which skips $k$ layers, the output contains information from both former layer and the skip connection, the output of this layer thus is presented by a set of jointly distributed random variables $(H_l,\mathnormal{C})$ ($1<l\leq L$, $1 \leq k \leq l-1$):
$$(H_l,\mathnormal{C}) \sim p_{\Phi}((h_l,\mathnormal{c}(h_{l-k}))|x,h_{1},\dots,h_{l-k},\dots,h_{l-1}) $$
We analyze the Fisher Information of output in order to find out if skip connections contribute to information preservation. The evaluation of Fisher Information thus becomes $\mathcal{I}_{(h_l,\mathnormal{c}(h_{l-k}))}\left(\phi_l\right)$. Since Fisher Information is always greater than or equal to zero, we can expand as follow by its chain rule \citep{Zegers2015} and deduce Proposition \ref{rmk:sc}:
\begin{align}\label{eq:chain}
\begin{split}
    \mathcal{I}_{(h_l,\mathnormal{c}(h_{l-k}))}\left(\phi_l\right) &= \mathcal{I}_{h_l}\left(\phi_l\right) + \mathcal{I}_{\mathnormal{c}(h_{l-k})|h_l}\left(\phi_l\right) \\
                                                &\geq \mathcal{I}_{h_l}\left(\phi_l\right) 
\end{split}
\end{align}

\begin{proposition}\label{rmk:sc}
Suppose the output of the $l^{th}~(l>1)$ hidden layer parameterized by $\phi_l$ receives information $f(h_{l-1})$ from the former layer $h_{l-1}$ and outputs the distribution: 
$$p_{\phi_l}(h_l|f(h_{l-1})) = p_{\phi_l}(h_l|x, h_{1:l-1}) $$. 

Modeling with Fisher Information, when connected with skip connections, this layer shall receive more information compared with non-skip architecture:
\begin{align}\label{eq:info}
\begin{split}
    \mathcal{I}_{(h_l,\mathnormal{c}(h_{l-k}))}\left(\phi_l\right) &= \mathcal{I}_{h_l}\left(\phi_l\right) + \mathcal{I}_{\mathnormal{c}(h_{l-k})|h_l}\left(\phi_l\right) \\
                                                &> \mathcal{I}_{h_l}\left(\phi_l\right) 
\end{split}
\end{align}
where the skip connection $c$ passes information $\mathnormal{c}(h_{l-k})$ from layer $h_{l-k}$ by skipping $k$ layers ($1 \leq k \leq l-1$). 
\end{proposition}
\begin{proof}
We only need to prove that $\mathcal{I}_{\mathnormal{c}(h_{l-k})|h_l}\left(\phi_l\right)$ is not zero in Eq. \eqref{eq:chain}. According to the theorem of chain rule \citep{Zegers2015}, for the inequality in Eq. \eqref{eq:chain}, the equality sign holds if and only if $h_l$ and $c(h_{l-k})$ are independent: $$\mathcal{I}_{\mathnormal{c}(h_{l-k})|h_l}\left(\phi_l\right) = 0 \Longleftrightarrow h_l \independent c(h_{l-k})$$ 
We have $h_{l} \sim p_{\theta}(h_{l}|x,h_{l},h_2,\dots,h_{l-k},\dots,h_{l-1})$ in Eq. \eqref{eq:subnet}, which indicates that $h_l$ and $c(h_{l-k})$ are not independent because they are both dependent to  $h_{l-k}$ and we have:
$$\mathcal{I}_{\mathnormal{c}(h_{l-k})|h_l}\left(\phi_l\right) > 0 \Longleftrightarrow  \mathcal{I}_{(h_l,\mathnormal{c}(h_{l-k}))}\left(\phi_l\right) > \mathcal{I}_{h_l}\left(\phi_l\right)$$
\end{proof}
Thanks to Proposition \ref{rmk:sc}, the skip connection can be regarded as a complementary information flow between layers. 
Following this idea, we propose a VAE model equipped with skip connections, named SCVAE. We make SCVAE skip one or more layers to keep information amount as rich as possible. Our model with one-layer skipping connections can be described by the following equations:
\begin{align}
    \begin{split}
    &h_l^{(en)} = f_l^{(en)}\left(h_{l-1}^{(en)}\right) + g\left(h_{l-1}^{(en)}\right) \\
    &z \sim \mathcal{N}(\mu(h_L),\sigma(h_L)) \\
    &h_l^{(de)} = f_l^{(de)}\left(h_{l-1}^{(de)}\right) + g\left(h_{l-1}^{(de)}\right)
    \end{split}
\end{align}
where $f_l$ indicates the $l^{th}$ layer's mapping in the neural network, $g$ is a is a down-sampling or up-sampling function, $h_0^{(en)} = x_{data}$, $h_0^{(de)} = z$, and $L$ is the depth of inference network. The model could also include long-skipping-distance connections which skip multiple layers to strengthen the sharing between low-level and high-level features, described as:
\begin{equation}
    h_l^{(nn)} = f_l^{(nn)}\left(h_{l-1}^{(nn)}\right) + g\left(h_{l-k}^{(nn)}\right)
\end{equation}
where $l-1 > l-k > 0$ and $(nn)$ refers to encoder/decoder.

In this way, SCVAE is essentially designed for information preservation. Skip connections as a simple method to preserve information flow, do not increase computation complexity and is compatible with many models as shown in experiments.


\section{Experimental Results}
In the following, we implement the experiments to verifying below three questions:
\begin{itemize}
    \item Whether the skip connections contribute to the information preservation in VAE models, especially in deep VAE architecture.
    \item Whether the observed degeneration gets mitigated as information is preserved, thus improves VAE model's performance when going deeper.
    \item Whether SCVAE is compatible with other method to reach a further amelioration. 
\end{itemize}

The experiments are conducted on the MNIST dataset that consists of ten categories of 28$\times$28 hand-written digits. We follow the standard split 50,000/10,000/10,000 to partition the dataset as the training, validation and test parts. 

Plain VAE and SCVAE are implemented using MLPs with layers of 500 parameters. The shallow VAE model is of depth 1 hidden layer. When we make the model's encoder (resp. decoder) deeper, we note as q++ (resp. p++). Otherwise encoder has the same depth with decoder. We also use a model SCVAE-L which only modifies SCVAE to contain only one long-skipping-distance connection in encoder to demonstrate the effect of long-skipping-distance connection. For all experiments, the dimension of latent space is set to 50. Quantitative results are presented with averaged values and Fisher Information is computed as noted in \cite{Desjardins:2015:NNN:2969442.2969471}.

\subsection{Fisher Information Preservation}
As discussed in Section \ref{sec:forget}, the information loss is one principle factor that obstructs VAE from going deeper. In this part, we evaluate how information amount decays as VAE goes deeper and their corresponding changes in encoder and decoder. From this perspective, we present the problems that deep VAE faces and how SCVAE overcomes these limitations to make VAE models deeper.

In the first experiment, we investigate the impact of depth on information amount in VAE. We respectively make their encoder and decoder deeper and compute their mean Fisher Information among layers. Figure \ref{fig:fishier depth} presents how information amount varies w.r.t. VAE depth. When VAE extends either encoder or decoder to go deeper, the average information amount keeps a decreasing tendency. Different from plain VAE, SCVAE could generally maintain information amount close to the same level. Although these two models have similar amount of information in shallow architecture, as the model goes deeper, SCVAE remains much richer information amount than plain VAE.  
\begin{figure}[t]
 \centering
\begin{minipage}{.4\textwidth}
         \includegraphics[width=\textwidth]{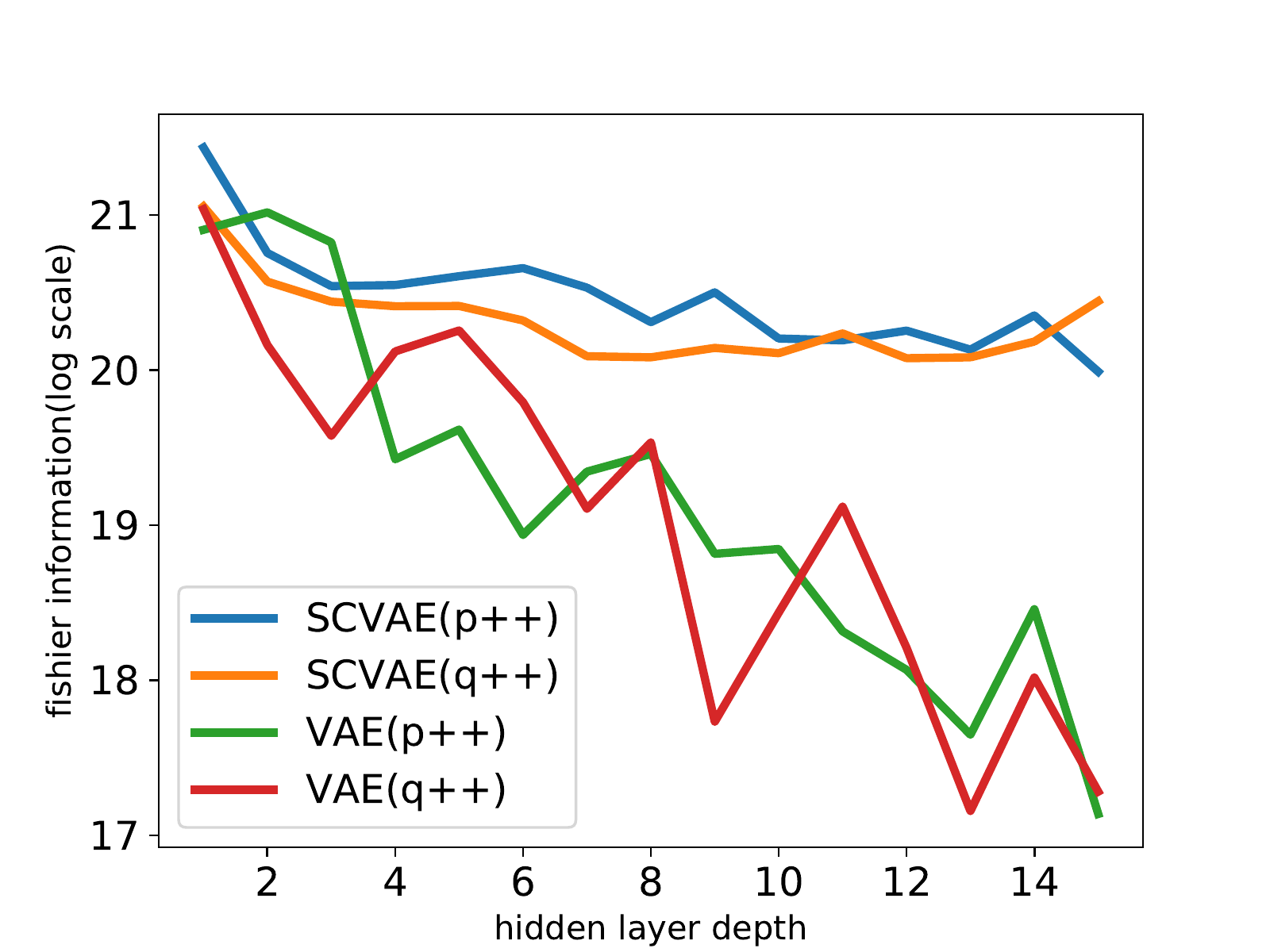}
    
     \caption{Mean Fisher Information w.r.t depth in VAE.}
     \label{fig:fishier depth}
\end{minipage}
\hfill
 \begin{minipage}{.4\textwidth}
         \includegraphics[width=\textwidth]{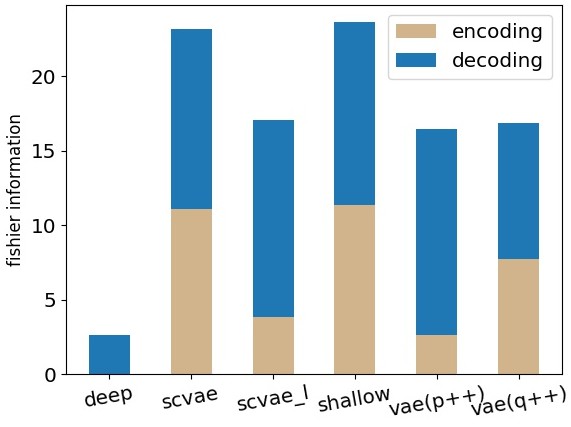}
     \caption{Mean Fisher Information in different VAE models.}
     \label{fig:fishier info}
\end{minipage}
     
\end{figure}

In the next experiment, we fix the depth for deepened part to 11-hidden-layer depth. For SCVAE-L, it keeps the same architecture as SCVAE, but only one long-skipping-distance connection exists in its encoder to connect the first and last hidden layer. We respectively compute the average Fisher Information in encoder and decoder, as shown in Figure \ref{fig:fishier info}. Deep VAE remains little information amount in the model, which corresponds to third type of degeneration mentioned in Section \ref{sec:forget}. When encoder goes deeper, information amount mainly decays in decoder; in reverse, information amount decays more in encoder when decoder goes deeper. This refers to the other two types of degeneration problems, indicating that blurry samples are caused by the lack of information in decoder, while abstruse latent representations are caused by the lack of information in encoder. In SCVAE-L, it is interesting that information amount in encoder is less than in SCVAE but richer than in VAE(p++), which implies that the model leverages the long-skipping-distance connection and augment the information amount in encoder but the capacity of preservation is finite. The advantages of skip connection in information preservation is shown in SCVAE, where we observe that SCVAE maintains the closest mean information amount to the shallow model, as well as the ratio between information amount in encoder and in decoder.

In these two experiments, we verify our claims in Section \ref{sec:forget} and \ref{sec:info}. When going deeper, VAE models tend to lose more information. We could thus associate information amount to phenomena in Figure \ref{motiv}. Information loss in decoder leads to blurry reconstruction samples, while loss in encoder leads to abstruse latent presentation. 

\subsection{Degeneration mitigation}
Previous experiments demonstrate VAE should carefully go deeper in case of information loss. In this part, we return to VAE tasks, \emph{i.e.}, representation learning and generation, in order to verify whether SCVAE mitigates the degeneration. We evaluate the representation learning with classification accuracy. A simple SVM (Support Vector Machine) is trained and test with the learned latent representation. As for generation, we evaluate with negative log-likelihood (NLL). All models keep the same architecture as in previous experiment.

Table \ref{tab:nll/clf} suggests going deeper results in an improvement in a specific task, though the other task performance suffers risk in degeneration: VAE(q++) achieves a better classification result, but sacrifices the NLL performance; VAE(p++) outperforms in NLL but under-performs in classification. Comparing to plain VAE, models with skip connections (SCVAE, SCVAE-L) achieve well performance in both tasks. Especially, SCVAE achieves the best result in both two tasks.

\begin{table}[tbp]
\centering
\caption{Test negative log-likelihood (NLL) and classification accuracy on MNIST}
\label{tab:nll/clf}
\begin{tabular}{lll}
Model & NLL & Acc \\ \hline
VAE(1L)\citep{journals/corr/KingmaW13}         & 87.89                 & 0.8421 \\
VAE(11L)        & 206.09                & 0.1135    \\
VAE(q++)  & 91.13                 &0.9352    \\
VAE(p++)  & 81.59                 & 0.7120     \\ \hline
SCVAE           & \textbf{80.19}                 & \textbf{0.9588}    \\
SCVAE-L     & 84.02                 & 0.9216    
\end{tabular}
\end{table}
\begin{figure}[htb]
\centering
        
        \includegraphics[width=.95\textwidth]{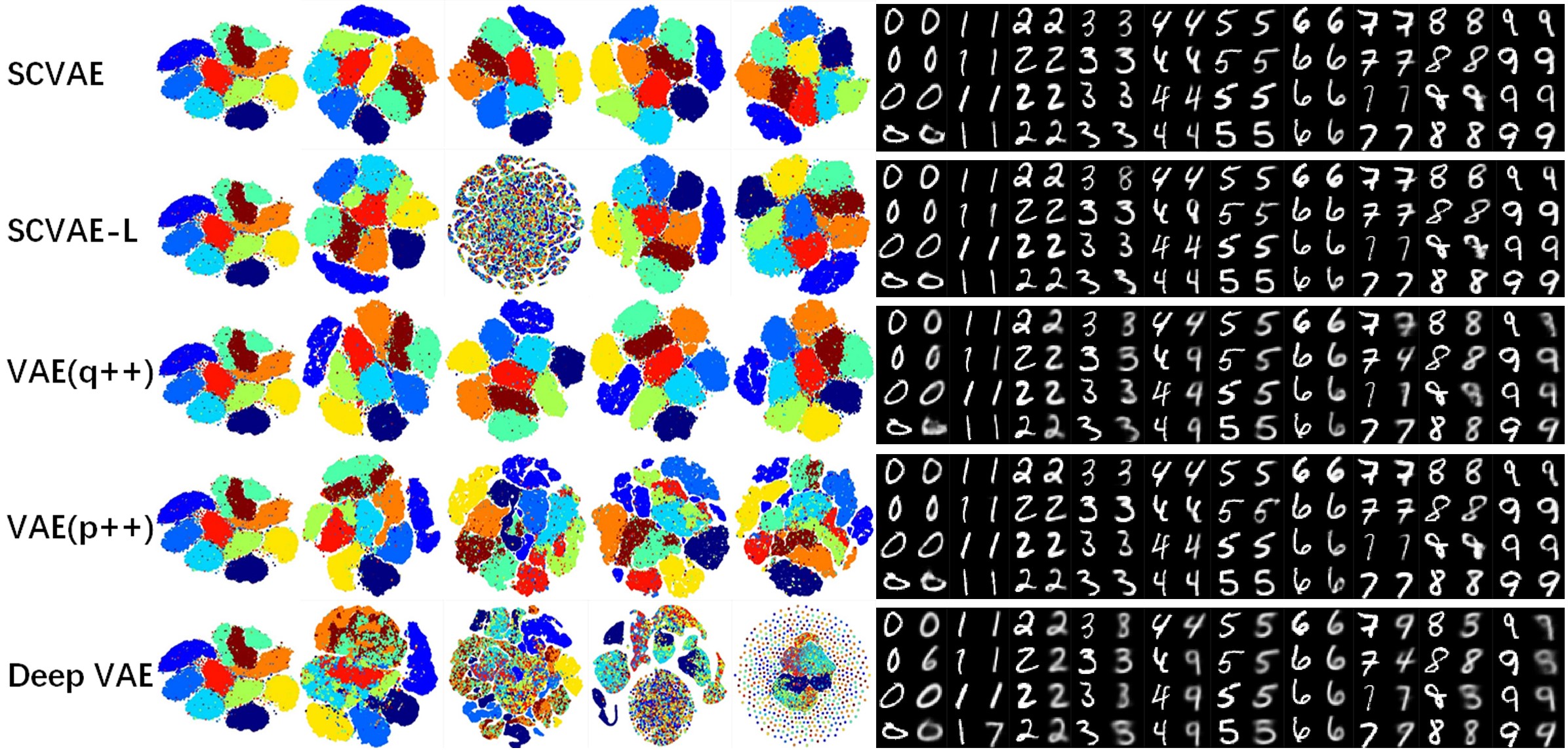}
        \label{fig:reconstruction}
    \caption{\textbf{Left}: representation visualization of raw data, first layer output, intermediate layer output, last layer output of encoder, and latent space (from left to right). \textbf{Right}: ground truth (odd columns) and reconstruction (even columns).}
    \label{fig:info evolution}
\end{figure}

In Figure \ref{fig:info evolution}, we present qualitative results of five deep models to have an intuitive understanding of the corresponding performance. As we analyzed in previous part, models suffer from degeneration due to the lack of information in encoder or decoder. When degenerated in encoder (VAE, VAE(p++)), the latent representation degenerates layer by layer in encoder; when degenerated in decoder (VAE, VAE(q++)), reconstructions are not only blurry but also contain incorrect digits. These phenomena explicates that the degenerated VAE does not connect the global and detailed information. When free from degeneration, SCVAE benefits from deep architecture to produce more clear reconstructions and to learn a compact representation. Recall that SCVAE-L contains less information in encoder than SCVAE and VAE(q++) (Figure \ref{fig:fishier info}), we notice that the intermediate layers show abstruse presentation, implying the information decays in these layers. 

In this experiment, we demonstrate that going deeper could benefit VAE models in specific task without information loss. Specifically, SCVAE is free from any type of degeneration, thus achieves the best performance among the previous models in both representation learning and generation.

\subsection{Combination with state-of-the-art}
To test the compatibility of deep SCVAE and other advancements in VAE, we respectively combine SCVAE with PixelVAE \citep{DBLP:journals/corr/GulrajaniKATVVC16} (with 8 pixel layers here) by concatenation, and with VampPrior \citep{tomczak2017vae} by substitution of Gaussian prior with VampPrior. SCVAE remains the same architecture as previous part.

In Figure \ref{fig:info pi}, we notice that our Fisher Information measure could also reflect the characteristics of state-of-the-art: we observe information in encoder is negligible comparing with decoder, which refers to the latent code ignorance problem in PixelVAE \citep{chen2016variational}. VampVAE has a more expressive posterior, thus performs better in latent coding and maintains more information in encoder. When combining with these methods, SCVAE improves their strength since SCVAE could be regarded as a powerful approximator for posterior and likelihood; SCVAE also remedies their shortcoming to a certain extent, by providing with more information amount. 

Table \ref{tab:combine} suggests that SCVAE has a comparable performance with the state-of-the-art. When combining with VampPrior and PixelVAE, SCVAE reinforces the power of these models, achieving a better performance than before. Notably, when combine all these methods, the performance of the final model becomes promising and competitive in both representation learning and generation. 

\begin{figure}[t]
\centering
    \begin{minipage}[t]{.4\columnwidth}
\captionof{figure}{Fisher Information measure in advanced models.}\label{fig:info pi}
        \includegraphics[width=\textwidth]{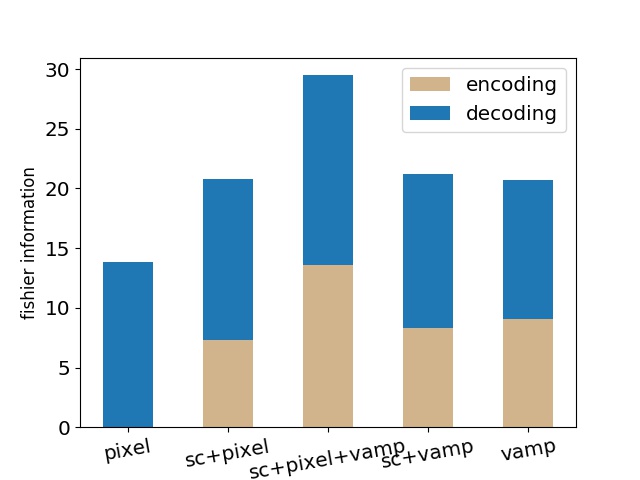}     
    \end{minipage}
    \hfill
    \begin{minipage}[t]{.58\columnwidth}
        \captionof{table}{Combination of SCVAE with state-of-the-art: Negative Log-Likelihood and classification accuracy on MNIST}
        \label{tab:combine}
        \begin{tabular}{lll}
        Model & NLL & Acc\\ \hline
        SCVAE                           &80.19   &0.9588 \\
        PixelVAE                        &79.48   &0.5148 \\
        VAE(1L) + VampPrior  &82.32 &0.9628 \\\hline
        SCVAE+VampPrior                 &81.63   &\textbf{0.9839}\\
        SCVAE + PixelVAE                &79.35     &0.7776 \\
        SCVAE + PixelVAE + VampPrior    &\textbf{79.26}      &0.9784
        \end{tabular}

    \end{minipage}

\end{figure}

\section{Conclusions}
In this paper, we investigate how deep architecture affects VAE models. Our observation shows that deeper architecture does not always benefit VAE performance due to three types of degeneration. In further analysis with our Fisher Information measure, we discover that the information loss is ineluctable for feed-forward networks and harms deep VAE with degeneration problems. Moreover, skip connections are proved to contribute in the information preservation without changing parameter structure. We thus propose a class of VAEs enhanced by skip connection, named SCVAE for information preservation and degeneration mitigation.

The experiments demonstrate the following advantages of SCVAE: 1) SCVAE maintains richer information to avoid the degeneration when going deeper; 2) SCVAE takes advantage of going deeper and achieve better performance in VAE's tasks, such as representation learning, generation, etc.; 3) SCVAE is compatible with other advanced VAE models for further improvement. Hence, SCVAE is promising in deep VAE and could be regarded as an appropriate design of the model. 

\bibliography{acml18}






\end{document}